%% file: neurips_2025.tex
\documentclass{article}
 \PassOptionsToPackage{round}{natbib}

\usepackage[preprint]{neurips_2025}

\usepackage[utf8]{inputenc} 
\usepackage[T1]{fontenc}    
\usepackage{hyperref}       
\usepackage{url}            
\usepackage{booktabs}       
\usepackage{amsfonts}       
\usepackage{amsmath}
\usepackage{amsthm}
\usepackage{amssymb}
\usepackage{nicefrac}       
\usepackage{microtype}      
\usepackage{xcolor}         

\newtheorem{theorem}{Theorem}
\newtheorem{prop}{Proposition}
\newtheorem{corollary}{Corollary}
\newtheorem{lemma}{Lemma}

\newtheorem{assumption}{Assumption}
\usepackage{algorithm}
\usepackage{algpseudocode}
\usepackage{diagbox}
\usepackage{footnote}
\usepackage{natbib}
\usepackage{tikz}
\usepackage{cancel}
\usepackage{cleveref}
\input{bmacros}

\newcommand{\ee}{\varepsilon}
\newcommand{\ucb}{\mathrm{UCB}}
\newcommand{\Reg}{\mathrm{Reg}}

\title{Linear Bandits with Non-i.i.d. Noise}

\author{%
  Baptiste Abélès  \\
  University Pompeu Fabra\\
  \texttt{baptistabeles@gmail.com} \\
   \And
   Eugenio Clerico \\
   University Pompeu Fabra \\
   \texttt{eugenio.clerico@gmail.com} \\
     \And
   Hamish Flynn \\
   University Pompeu Fabra \\
   \texttt{hamishflynn.gm@gmail.com} \\
     \And
   Gergely Neu \\
   University Pompeu Fabra \\
    \texttt{gergely.neu@gmail.com} \\
}

\begin{document}

\maketitle

\begin{abstract}
 We study the linear stochastic bandit problem, relaxing the standard i.i.d.~assumption on the observation noise. 
As an alternative to this restrictive assumption, we allow the noise terms across rounds to be sub-Gaussian but 
interdependent, with dependencies that decay over time. To address this setting, we develop new confidence sequences 
using a recently introduced reduction scheme to sequential probability assignment, and use these to derive a bandit 
algorithm based on the principle of optimism in the face of uncertainty. We provide regret bounds for the 
resulting algorithm, expressed in terms of the decay rate of the strength of dependence between observations. Among 
other results, we show that our bounds recover the standard rates up to a factor of the mixing time for geometrically 
mixing observation noise.
\end{abstract}

\input{Introduction}

\input{Preliminaries}
\input{Mixing_Bandit}
\input{confidence_set}

\input{Regret_analysis}

\input{Conclusion}

\ack{This project has received funding from the European Research Council (ERC), under the European Union’s Horizon 2020 research and in- novation programme (Grant agreement No. 950180).}

\bibliography{biblio}



\newpage


\newpage
\appendix

\input{Appendix}

\end{document}

%% file: bmacros.tex
\newcommand{\Regret}{\mathrm{Regret}}

\newcommand{\dd}{\mathrm{d}}

 \newcommand{\X}{\mathcal{X}}

\newcommand{\F}{\mathcal{F}}

\newcommand{\real}{\mathbb{R}}
\newcommand{\C}{\mathcal{C}}

\newcommand{\OO}{\mathcal{O}}

\newcommand{\E}{\mathbb{E}}

\newcommand{\EEc}[2]{\mathbb{E}\left[#1\left|#2\right.\right]}

\def\argmin{\mathop{\mbox{ arg\,min}}}
\def\argmax{\mathop{\mbox{ arg\,max}}}

\newcommand{\siprod}[2]{\langle#1,#2\rangle}
\newcommand{\iprod}[2]{\left\langle#1,#2\right\rangle}

\newcommand{\norm}[1]{\left\|#1\right\|}

\newcommand{\pa}[1]{\left(#1\right)}

\newcommand{\wh}{\widehat}
\newcommand{\wt}{\widetilde}

\usepackage{todonotes}
\definecolor{PalePurp}{rgb}{0.66,0.57,0.66}

\usepackage[normalem]{ulem}
\usepackage{enumitem}

%% file: Introduction.tex
\section{Introduction}
The linear bandit problem \citep{abe1999associative, auer2003using} is an instance of a multi-armed bandit framework, 
where the expected reward is linear in the feature vector representing the chosen arm. More concretely, it is a 
sequential decision-making problem, where an agent each round picks an arm $X_t$, and receives a reward $Y_t = 
\siprod{\theta^\star}{X_t} + \ee_t$, with $\theta^\star$ a fixed parameter unknown to the agent, and $\ee_t$ zero-mean 
random noise. This framework has gained significant attention in the literature as it yields analytic tools that can be 
applied to several concrete applications, such as online advertising \citep{abe2003reinforcement}, recommendation 
systems \citep{li2010contextual, korkut2021disposable}, and dynamic pricing \citep{cohen2020feature}. 

A popular strategy to tackle linear bandits leverages the principle of \emph{optimism in the face of uncertainty}, via 
upper confidence bound (UCB) algorithms. The idea of optimism can be traced back to \citet{lai1985asymptotically}, and 
its application to linear bandits was already advanced by \citet{auer2003using}. Since then, this approach has been 
improved and analysed by several works \citep{abbasi2011improved, lattimore2020bandit, flynn2023improved}. This class of 
methods requires constructing an adaptive sequence of confidence sets that, with high probability, contain the true 
parameter $\theta^\star$. Each round, the agent selects the arm maximising the expected reward under the most optimistic 
parameter (in terms of reward) in the current confidence set. UCB-based algorithms have become popular as they are often 
easy to implement and come with tight worst-case regret guarantees. 

For a UCB algorithm to perform well, it is necessary that the confidence sets are tight, which can be ensured by taking 
advantage of the structure of the problem. In this paper, our focus is on studying various assumptions on the 
observation noise. A commonly studied situation is when $(\ee_t)_{t\geq 0}$ consists of a sequence of 
i.i.d.~realisations of some bounded or sub-Gaussian random variable (see \citealp{lattimore2020bandit}, Chapter 20). 
Often, the standard analysis can be extended to the case in which the realisation are not independent, but conditionally 
centred and sub-Gaussian \citep{abbasi2011improved}. Yet, in real-world settings, this assumption is often unrealistic, 
as one can expect the presence of interdependencies among the noise at different rounds. For instance, in the context of 
advertisement selection, the noise models the ensemble of external factors that influence the user's choice on whether 
to click or not an ad. The i.i.d.~assumption implies that across different rounds these external factors are completely 
independent. In practice, the user choice will be affected by temporally correlated events, such as recent browsing 
history or exposure to similar content. Therefore, a more realistic assumption is to allow the dependencies to decay 
with time, rather than being completely absent. This way to model dependencies, often referred to as \emph{mixing}, is common to study concentration for sums of non-i.i.d.~random variables, with applications to machine learning \citep{bradley2005basic, mohri2008rademacher,abeles2024generalization}.

In the present paper we relax the assumption that the noise is conditionally zero-mean in the bandit problem, and we allow for the presence of dependencies. Concretely, we replace  the standard 
conditionally sub-Gaussian setting with a more general formulation that accounts for conditional 
dependence of the noise on the past, by introducing a natural notion of \emph{mixing sub-Gaussianity}. Within this 
context, we introduce a UCB algorithm for which we rigorously establish regret guarantees. There are two key challenges 
for our approach: constructing a valid confidence sequence under dependent noise, and deriving a regret upper bound for 
the UCB algorithm that we propose. 

We derive the confidence sequence by adapting the \emph{online-to-confidence-sets} technique to accommodate temporal 
dependencies in the noise. This approach, originally introduced by \cite{abbasi2011improved} and recently extended and 
improved \citep{jun2017scalable, lee2024improved, clerico2025confidence}, involves 
constructing an abstract online learning game whose regret guarantees can be turned into a confidence sequence. To deal 
with the dependencies in the noise, we modify the standard online-to-confidence-sets framework by introducing delays in 
the feedback received within the abstract online game. This approach is inspired by the recent work of 
\cite{abeles2024generalization} on extending online-to-PAC conversions to non-i.i.d.~mixing data sets in the context of 
deriving generalisation bounds for statistical learning. There, a  delayed-feedback trick similar to ours is employed to 
derive statistical guarantees (generalisation bounds) from an abstract online learning game. 

For the regret analysis of the bandit algorithm, we also need to face some challenges due to the correlated 
observation noise. We address these by introducing delays into the decision-making policy as well.
This makes our approach superficially similar to algorithms used in the rich literature on bandits with delayed feedback
(see, e.g., \citealp{vernade2020linear,howson2023delayed}). These works consider delay as part of the problem statement 
and not part of the solution concept, and are thus orthogonal to our work. In particular, a simple adaptation of 
results from this literature would not suffice for dealing with dependent observations, which we tackle by developing 
new concentration inequalities. Another line of work that is conceptually related to ours is that of non-stationary 
bandits \citep{garivier2008upper,russac2019weighted}. In that setting, the parameter vector $\theta^\star_t$ evolves in 
time according to a nonstationary stochastic process, and the observation noise remains i.i.d., once again making for a 
rather different problem with its own challenges. Namely, the main obstacle to overcome is that comparing with the 
optimal sequence of actions becomes impossible unless strong assumptions are made about the sequence of 
parameter vectors. A typical trick to deal with these nonstationarities is to discard old observations (which may have 
been generated by a very different reward function), and use only recent rewards for decision-making. This is the polar 
opposite of our approach that is explicitly \emph{disallowed} to use recent rewards, which clearly highlights how 
different these problems are. That said, there exists an intersection between the worlds of delayed and nontationary 
bandits \citep{vernade2020non}, and thus we would not discard the possibility of eventually building a bridge between 
bandits with nonstationary reward functions and bandits with nonstationary observation noise. For simplicity, we focus 
on the second of these two components in this paper.

%

 
 \paragraph{Notation.}
 Throughout the paper, we will often use the following notations. For $u$ and $v$ in $\real^p$, we let $\langle u, v\rangle$ denote their dot product. $\|u\|_2 = \sqrt{\langle u, u\rangle}$ is the Euclidean norm, while for a non-negative definite $(p\times p)$-matrix $A$, $\|u\|_A = \sqrt{\langle u, Au\rangle}$ is a semi-norm (a norm if the matrix is strictly positive definite). For $r>0$, $\mathcal B(r)$ denotes the closed centred Euclidean ball in $\real^p$ with radius $r$. Given a non-empty set $U\subseteq\real^p$, we let $\Delta_U$ denote the space of (Borel) probability measures on $\real^p$ whose support in $U$. Finally, $(u_t)_{t\geq t_0}$ denotes a sequence indexed on the integers, with $t_0$ its smallest index.

%% file: Preliminaries.tex


\section{Preliminaries on linear bandits}
\label{subsec:LSB}
We consider a version of the classic problem of regret minimisation in stochastic linear bandits, where an agent needs to 
make a sequence of decisions (or pick an \emph{arm}) from a given contextual decision set that may change over the 
sequence of rounds. We assume that the environment is oblivious to the actions of the 
agent, in the sense that the decision sets are determined in advance, and do not depend neither on the realisations of 
the noise nor on the agent's arm-selection strategy. 

Concretely, we define the problem as follows. Let $\theta^\star\in\real^p$ be a parameter vector that is unknown to the 
learning agent. We assume as known an upper bound $B>0$ on its euclidean norm (namely, $\theta^\star\in\mathcal B(B)$). Fix 
a sequence of decision sets $(\X_t)_{t\geq 1}$ in $\real^p$. We assume that for all $t$ we have $\X_t\subseteq\mathcal B(1)$.  At each round $t$, the agent is required to pick an arm 
$X_t\in\X_t$, and receives the reward $Y_t = \siprod{\theta^\star}{X_t} + \ee_t$. The sequence $(\ee_t)_{t\geq1}$ 
represents the random feedback noise. The noise across different rounds is typically assumed to be conditionally centred 
and to have well behaved tails. For instance, a common assumption is to ask that $\E[\ee_t|\F_{t-1}]$ 
is centred and 
sub-Gaussian, where $\F_{t} = \sigma(\ee_1,\dots, \ee_t)$ is the $\sigma$-field generated by the noise. This is the assumption this work relaxes. We also remark that, more generally, one can consider the case where the $X_t$ as well are randomised, namely contain additional randomness that is not included in the noise. To take this into account, one can add this other source or randomness in the filtration. However, since in our case we will only consider a non-randomised bandit algorithm, we omit this to simplify our analysis.

The agent aims to find a good strategy to pick arms $X_t$ that lead to a high expected $T$-round reward 
$\sum_{t=1}^T\siprod{X_t}{\theta^\star}$. To compare their performance to that of an agent playing each round the best 
available arm (in expectation), we define the \emph{regret} after $T$ rounds as
\[\Reg(T) =\sum_{t=1}^t \sup_{x\in\X_t}\big( \siprod{x}{\theta^\star} - \siprod{X_t}{\theta^\star}\big)\,.\]

A common approach to tackle the linear bandit problem is to follow an \emph{upper confidence bound} (UCB) strategy. This 
involves the following protocol. At each round $t$, we first derive a confidence set $\C_{t-1}$, based on the 
arm-reward pairs $(X_s, Y_s)_{s\leq t-1}$. This is a random set (as it depends on the past noise realisations), which 
must be constructed ensuring that $\theta^\star\in \C_{t-1}$ with high probability. More precisely, the regret can be 
effectively controlled if one  can ensure that $\theta^\star$ uniformly belongs to every set $(\C_t)_{t\geq 1}$, with 
high probability (a property often referred to as \emph{anytime validity}). Then, for every available arm $x$, 
we let $$\ucb_{\C_{t-1}}(x) = \max_{\theta\in \C_{t-1}}\,\siprod x\theta\,.$$ By definition, this is a high-probability 
upper bound on $\siprod x{\theta^\star}$, which justifies the name ``upper confidence bound''. The idea is then to 
\emph{optimistically} pick as $X_t\in\X_t$ the arm maximising $\ucb_{\C_{t-1}}$.



A key technical challenge in designing a UCB algorithm is to construct the anytime valid confidence sequence $(\C_t)_{t\geq 1}$. Typically, under sub-Gaussian assumptions on the noise, these sets take the form of an ellipsoid, centred on a (regularised) maximum likelihood estimator. Explicitly, we often have
$$\C_t = \big\{\theta\in\Theta\,:\,\|\theta - \widehat{\theta}_{t}\|_{V_{t}}^2 \leq \beta_t^2\big\}\,,$$
where $\wh{\theta}_t$ is the least-squares estimator of $\theta^\star$, $V_{t}$ is the \emph{feature-covariance} matrix 
and $\beta_t$ is a radius carefully chosen so that the high-probability coverage 
requirement is satisfied. In this work, to construct the confidence sets we will leverage an 
\emph{online-to-confidence-set-conversion} approach, a method that reduces the problem of proving statistical 
concentration bounds to proving existence of well-performing algorithms for an associated game of \emph{sequential 
probability assignment}. We refer to \Cref{sec:conf_sets} for more details on our technique to construct the confidence 
sequence. 

%% file: Mixing_Bandit.tex
\section{Linear bandits with non-i.i.d.~observation noise} \label{sec:contrib} 
We study a variant of the standard linear stochastic bandit problem where the observation-noise variables feature 
dependencies across different rounds. We focus on the case of weakly stationary noise, meaning we assume all the $\ee_t$ to have the same marginal distribution. 
However, the core assumption we make is what we call \emph{mixing sub-Gaussianity}. This provides a way to control how dependencies decay as the time between two 
observations increases. It is defined in terms of a sequence of mixing coefficients $\phi_d$, which quantify this decay. \begin{assumption}[Mixing sub-Gaussianity]
\label{ass:mixing-subgaussianity}
Fix $\sigma>0$ and let $\phi = (\phi_d)_{d\geq 0}$ be a non-negative and non-increasing sequence. We say that the random sequence $(\epsilon_t)_{t\geq 1}$ is \emph{$(\sigma, \phi)$-mixing sub-Gaussian} if $\ee_t$ is centred and $\sigma$-sub-Gaussian for every $t$, and, for all $d\geq 0$ and all $t>d$, we have
\begin{equation}
\label{eq::mixing}
    \bigl|\EEc{\epsilon_t}{\F_{t-d}} \bigr| \leq \phi_d
\end{equation}
and
\begin{equation}
\label{eq::mixing-subgaussianity}
   \EEc{\exp{\lambda(\epsilon_t-\EEc{\epsilon_t}{\F_{t-d}})}}{\F_{t-d}}\leq e^{\frac{\lambda^2\sigma^2}{2}}\,,
\qquad \forall \lambda>0\,.
\end{equation}
\end{assumption}

Clearly, the above assumption generalises the standard conditionally sub-Gaussian assumption (that can be recovered by 
setting $\phi_d=0$ for all $t$), sometimes considered in the bandit literature. Although this might look like an 
unusual mixing assumption, it is very natural for our problem at hand, and can be weaker than standard mixing 
hypotheses. For instance, if the noise sequence is $\varphi$-mixing (see \citealp{bradley2005basic}) and each $\ee_t$ 
is centred and bounded in $[-a,b]$, it is straightforward to check that $|\E[\ee_t|\F_{t-d}]| \leq (a+b)\phi_d$, and so 
Assumption \ref{ass:mixing-subgaussianity} is satisfied since the boundedness automatically implies sub-Gaussianity. In the rest of the paper we assume $\sigma =1$ for simplicity.

Under Assumption \ref{ass:mixing-subgaussianity}, we can build the confidence sequence needed for our UCB algorithm. We 
state this result below, but defer the explicit derivation to \Cref{sec:conf_sets} (see Corollary \ref{cor:conf} there). 

\begin{prop}
 \label{prop:CS_Mixing_Bandit} For some given $\phi$, let the noise satisfy Assumption \ref{ass:mixing-subgaussianity} with $\sigma=1$. Fix $\delta\in(0,1)$, $\lambda>0$, and $d\geq 1$. For $t\geq 1$ let $$\mathcal\C_t = \left\{\theta\in\mathcal B(B)\,:\,\tfrac{1}{2}\|\theta-\wh\theta_t\|_{V_t}^2\leq \tfrac{dp}{2}\log\tfrac{(B+1)^2e\max(dp,t+d)}{dp} + 2\lambda B^2 + t\phi_d (B+1) + d\log\tfrac{d}{\delta}\right\}\,,$$ where $V_t = \sum_{s=1}^t X_tX_t^\top + \lambda\mathrm{Id}$, and $\wh\theta_t = \argmin_{\theta\in\mathcal B(B)} \sum_{s=1}^t(\langle \theta,X_t\rangle -Y_t)^2$. Then, $(\C_t)_{t\geq 1}$ is an anytime valid confidence sequence, in the sense that
 $$\mathbb P\big(\theta^\star\in \C_t\,,\;\forall t\geq 1\big)\geq 1-\delta\,.$$
 \end{prop}
 

Leveraging the confidence sequence above, we can define a UCB approach for our problem (\Cref{alg::Mixing_LinUCB}). At a high level, the algorithm 
 operates by taking the confidence sets defined in Proposition 
\ref{prop:CS_Mixing_Bandit}, and selecting the arm optimistically, as in the standard UCB. A key point is that a delay 
$d$ is introduced, which at round $t$ restricts the agent to use only the information available from the first $t-d$ 
rounds. Although the actual technical reason behind this restriction will become fully clear only with the analysis of the 
coming sections, one can intuitively think of it as a way to prevent overfitting to recent noise, which might be 
highly correlated. If $d$ is sufficiently large, the noise observed in each round $t$ will 
be sufficiently decorrelated from the previous observations, which allows accurate estimation and uncertainty 
quantification of the true parameter $\theta^\star$ and the associated rewards.  

\begin{algorithm}[!h]
\caption{Mixing-LinUCB}\label{alg::Mixing_LinUCB}
\begin{algorithmic}
\State set $d>0$
\For {$i \in \{1,2,\dots d\}$}
\State play an arbitrary $X_i$ and observe $Y_i$
\EndFor
\For {$t \in \{d+1,\dots\}$}
\State  $X_{t} = \arg\max_{x \in \X_t} \ucb_{\C_{t-d}}(x)$, where $\C_{t-d}$ is as in Proposition \ref{prop:CS_Mixing_Bandit}
\State play $X_t$ and observe reward $Y_t$
\EndFor
\end{algorithmic}
\end{algorithm} 

In Section \ref{sec:bandit_regret} we provide a detailed analysis of the regret of the algorithm that we proposed. For instance, assuming that the mixing coefficients decay exponentially as $\phi_d = C e^{-d/\tau}$ (\emph{geometric mixing}), we show that the regret can be upper bounded in high probability as  \[\Reg(T) \leq \OO\left(\tau p\sqrt{T}\log(T)^2+\tau \log T\sqrt{pT \log T}\right).\]
We refer to Theorem \ref{thm:worst_case_reg} and Corollary \ref{cor:geometric_mixing} in Section \ref{sec:bandit_regret} for more details.

%% file: confidence_set.tex
\section{Constructing the confidence sequence}
\label{sec:conf_sets}

In this section we derive a confidence sequence for linear models with non-i.i.d.~noise. First, we briefly describe the online-to-confidence-set conversion scheme from \citet{clerico2025confidence}, which serves as our starting point. We then extend this technique to handle mixing noise.

\subsection{Online-to-confidence set conversion for i.i.d.~data}
\label{sec:conf_set_proof}
Before proceeding for the analysis of mixing sub-Gaussian noise, which is the focus of this work, we start by describing how to derive a confidence sequence when the noise is independent (or conditionally) centred and sub-Gaussian across different rounds, as in \cite{clerico2025confidence}. The online-to-confidence sets framework that we consider instantiates an abstract game played between an \emph{online learner} and an \emph{environment}. We define the squared loss $\ell_s(\theta) = \frac{1}{2}(\iprod{\theta}{X_s} - Y_s)^2$. For each round $s = 1, \dots, t$, the following steps are repeated:
\begin{enumerate}
\item the environment reveals $X_s$ to the learner;
\item the learner plays a distribution $Q_s \in \Delta_{\mathbb{R}^p}$;
\item the environment reveals $Y_s$ to the learner;
\item the learner suffers the log loss $\mathcal{L}_s(Q_s) = -\log\int_{\mathbb{R}^p}\exp(-\ell_s(\theta))\mathrm{d}Q_s(\theta)$.
\end{enumerate}

This game is a special case of a well-studied problem called \emph{sequential probability assignment} \citep{cesabianchi2006prediction}. The learner can use any strategy to choose $Q_1, \dots, Q_t$, as long as each $Q_s$ depends only on $X_1, Y_1, \dots, X_{s-1}, Y_{s-1}, X_s$. We define the \emph{regret} of the learner against a (possibly data-dependent) comparator $\bar{\theta} \in \mathbb{R}^p$ as
\begin{equation*}
\Regret_t(\bar{\theta}) = \sum_{s=1}^{t}\mathcal{L}_s(Q_s) - \sum_{s=1}^{t}\ell_s(\bar{\theta})\,.
\end{equation*}

\cite{clerico2025confidence} provide a regret bound upper bound (Proposition 3.1 there) for when the learner's strategy is from an \emph{exponential weighted average} (EWA) forecaster with a centred Gaussian prior $Q_1$. However, to account for the presence of dependencies in our analysis, we will need the prior's support to be bounded. We hence state here a regret bound (whose proof is deferred to Appendix \ref{app:prop2}) for the regret of an EWA forecaster with a uniform prior. 


\begin{prop}
Fix $B>0$ and consider the EWA forecaster with as prior the uniform distribution on $\mathcal B(B+1)$. Then, for all $\bar{\theta}\in\mathcal B(B)$ and any $t\geq 1$, 
\begin{equation*}
\Regret_t(\bar{\theta}) \leq \frac{p}{2}\log\frac{(B+1)^2e\max(p,t)}{p}\,.
\end{equation*}
\label{pro:uniform_ewa_regret}
\end{prop}

We remark that, by adding and subtracting the total log loss of the learner, the excess loss of $\theta^\star$ (relative to $\bar{\theta}$) can be rewritten as
\begin{equation}\label{eq:otc_dec}
\sum_{s=1}^t\ell_s(\theta^{\star}) - \sum_{s=1}^t\ell_s(\bar{\theta}) = \Regret_t(\bar{\theta}) + \sum_{s=1}^t\ell_s(\theta^{\star}) - \sum_{s=1}^{t}\mathcal{L}_s(Q_s)\,.
\end{equation}
This simple decomposition is the key idea in the online-to-confidence sets scheme. 

Since the noise is conditionally sub-Gaussian and the distributions played by the online learner are predictable ($Q_s$ cannot depend on $Y_s$), $\sum_{s=1}^t\ell_s(\theta^{\star}) - \sum_{s=1}^{t}\mathcal{L}_s(Q_s)$ is the logarithm of a non-negative super-martingale (cf. the no-hypercompression
inequality in \citealp{grunwald2007minimum} or Proposition 2.1 in \citealp{clerico2025confidence}) with respect to the noise filtration $(\F_t)_{t\geq 1}$. For simplicity, as already mentioned in Section \ref{subsec:LSB} and since this will be the case for our bandit strategy, we assume throughout the paper that $X_t$ is fully determined given the past noise . Henceforth, from Ville's inequality (a classical anytime valid Markov-like inequality that holds for non-negative super-martingales) one can easily derive that $\theta^\star\in\C_t$ (uniformly for all $t$) with probability at least $1-\delta$, where
$$\mathcal{C}_t = \left\{\theta \in \mathbb{R}^p: \sum_{s=1}^{t}\ell_s(\theta) - \sum_{s=1}^{t}\ell_s(\bar{\theta}) \leq \Regret_t(\bar{\theta}) + \log\frac{1}{\delta}\right\}\,.$$
This result can be relaxed by replacing $ \Regret_t(\bar{\theta})$ by any known regret upper bound for the online algorithm used in the abstract game (\emph{e.g.}, the bound of Proposition \ref{pro:uniform_ewa_regret} for the EWA forecaster).

\subsection{Confidence sequence under mixing sub-Gaussian noise}
The standard online-to-confidence sets scheme relies on the fact that $\sum_{s=1}^t\ell_s(\theta^{\star}) - \sum_{s=1}^{t}\mathcal{L}_s(Q_s)$ is the logarithm of a non-negative super-martingale, whose fluctuations can be controlled uniformly in time thanks to Ville's inequality. However, this property hinges on the fact that the noise is assumed to be conditionally centred and sub-Gaussian, which now is not anymore the case. Yet, thanks to our mixing assumption, if we restrict our focus on rounds that are sufficiently far apart, the mutual dependencies get weaker, and the exponential of the sum behaves \emph{almost} like a martingale. This insight suggests to partition the rounds into blocks, whose elements are mutually far apart, then apply concentration results to each block, and finally use a union bound to recover the desired confidence sequence spanning all rounds. We remark that this is a classical approach to derive concentration results for mixing processes, often referred to as the \emph{blocking} technique \citep{yu1994rates}. 

In order for the online-to-confidence sets scheme to leverage the blocking strategy outlined above, the abstract online game used for the analysis must be designed in a way that is compatible with the block structure. To address this point, we adopt an approach inspired by \cite{abeles2024generalization}, who introduced delays in the feedbacks received by the online learner in order to address a similar challenge. More precisely, we will now consider the following \emph{delayed-feedback} version of the online game. Fix a delay $d>0$. For each round $s = 1, \dots, t$, the following steps are repeated:
\begin{enumerate}
\item the environment reveals to the learner $X_s$, which is assumed to be $\F_{s-d}$-measurable;
\item the learner plays a distribution $Q_s \in \Delta_{\mathbb{R}^p}$;
\item if $s>d$, the environment reveals $Y_{s-d+1}$ to the learner;
\item the learner suffers the log loss $\mathcal{L}_s(Q_s) = -\log\int_{\mathbb{R}^p}\exp(-\ell_s(\theta))\mathrm{d}Q_s(\theta)$.
\end{enumerate}
Note that the delay $d$ only applies for the rewards, while $Q_s$ can still depend on $X_s$. Indeed, the choice of $X_s$ in our mixing UCB algorithm is already ``delayed'', as it depends on $\C_{t-d}$ (see Algorithm \ref{alg::Mixing_LinUCB}).

Of course, in this setting the decomposition of \eqref{eq:otc_dec} is still valid. We now want to deal with the concentration of $\sum_{s=1}^t\ell_s(\theta^{\star}) - \sum_{s=1}^{t}\mathcal{L}_s(Q_s)$ via the blocking technique. For convenience, let us write $D_t = \ell_t(\theta^\star) - \mathcal L_t(Q_t)$. We denote as $S^{(i)} = (S_k^{(i)})_{k\geq 1}$ the subsequence defined as $S_k^{(i)} = \sum_{j=1}^k D_{i+(j-1)d}$. The key idea is now that each of these $S^{(i)}$ behaves as the log of a martingale, up to a cumulative remainder that accounts for the conditional mean shift in the mixing sub-Gaussianity assumption. In particular, Ville's inequality and a union bound yield the following. 

\begin{lemma}\label{lemma:conc}
	Fix a delay $d>0$ and $\delta\in(0,1)$. We have that
	$$\mathbb P\left(\sum_{s=1}^t\big(\ell_s(\theta^{\star}) - \mathcal{L}_s(Q_s)\big) \leq t\phi_d B + d\log\frac{d}{\delta}\,,\;\forall t\geq 1\right)\geq 1-\delta\,.$$
\end{lemma}

Now that we have a concentration result to control $S_t$, we only need to be able to upper bound the regret of an algorithm for the ``delayed'' online game that we are considering. To this purpose, we propose the following approach. We run $d$ independent EWA forecaster (with uniform prior), each one only making prediction and receiving the feedback once every $d$ rounds. More explicitly, the first forecaster acts at rounds $1$, $d+1$, $2d+1$..., the second at round $2$, $d+2$, $2d+2$..., and so on. As a direct consequence of Proposition \ref{pro:uniform_ewa_regret}, by summing the individual regret upper bounds we get a regret bound for the joint forecaster, which at each round returns the distribution predicted by the currently active forecaster. This technique of partitioning rounds into blocks for the regret analysis of online learning is common in the literature (\emph{e.g.}, see \citealp{weiberger2002delay}).

\begin{lemma}\label{lemma:reg}
Fix $B>0$, $d>0$, and consider a strategy with $d$ independent EWA forecasters outlined above, all initialised with the uniform distribution on $\mathcal B(B+1)$ as prior. For all $\bar{\theta}\in\mathcal B(B)$ and $t\geq 1$,
\begin{equation*}
\Regret_t(\bar{\theta}) \leq \frac{dp}{2}\log\frac{(B+1)^2e\max(dp,t+d)}{dp}\,.
\end{equation*}
\end{lemma}

Putting together what we have, we get a confidence sequence suitable for our mixing UCB algorithm. 

\begin{theorem}\label{thm:confseq}
	Consider the setting introduced above. Fix $\delta\in(0,1)$ and a delay $d>0$. Assume as known that $\theta^\star\in\mathcal B(B)$. Let $\wh\theta_t = \argmin_{\theta\in\mathcal B(B)}\{\sum_{s=1}^{t}\ell_s(\theta)\}$ and $\Lambda_t = \sum_{s=1}^t X_s X_s^\top$. Define 
	$$\C_t = \left\{\theta\in\mathcal B(B)\,:\,\tfrac{1}{2}\|\theta-\wh\theta_t\|_{\Lambda_t}^2\leq \tfrac{dp}{2}\log\tfrac{(B+1)^2e\max(dp,t+d)}{dp} + t\phi_d (B+1) + d\log\tfrac{d}{\delta}\right\}\,.$$
	Then, $(\C_t)_{t\geq 1}$ is an anytime valid confidence sequence for $\theta^\star$, namely
	$$\mathbb P\big(\theta^\star\in\C_t\,,\;\forall t\geq 1\big)\leq 1-\delta\,.$$
\end{theorem}
\begin{proof}
	The optimality of $\wh\theta_t$ implies $\sum_{s=1}^t\langle \theta-\wh\theta_t,\nabla\ell_s(\wh\theta_t)\rangle \geq 0$, for all $\theta\in\mathcal B(B)$. As  $\sum_{s=1}^t\ell_s$ is quadratic, it equals its second order Taylor expansion around $\wh\theta_t$ and its Hessian is everywhere $\Lambda_t$. So,
	$$\frac12\|\theta-\wh\theta_t\|_{\Lambda_t}^2 \leq \frac12\|\theta-\wh\theta_t\|_{\Lambda_t}^2+\sum_{s=1}^t\big\langle \theta-\wh\theta_t,\nabla\ell_s(\wh\theta_t) \big\rangle= \sum_{s=1}^t \big(\ell_s(\theta)-\ell_s(\wh\theta_t)\big)\,,$$ for any $\theta\in\mathcal B(B)$. This, together with \eqref{eq:otc_dec}, Lemma \ref{lemma:conc}, and Lemma \ref{lemma:reg}, yields the conclusion.
\end{proof}
We remark that the confidence sets of Theorem \ref{thm:confseq} take the form of the intersection between the ball $\mathcal B(B)$ and the ``ellipsoid'' $\{\theta:\|\theta-\wh\theta_t\|_{\Lambda_t}\leq\beta_t\}$, for a suitable radius $\beta_t$. In order to implement and analyse the bandit algorithm, it will be more convenient to work with a relaxation of these sets, a pure ellipsoid not intersected with $\mathcal B(B)$. We make this explicit in the following corollary.
\begin{corollary}\label{cor:conf}
	Fix $\lambda>0$, $d>0$, and $\delta\in(0,1)$. For $t\geq 1$, let $V_t = \Lambda_t + \lambda\mathrm{Id}$. Assuming that $\theta^\star\in\mathcal B(B)$, the following compact ellipsoids define an anytime valid confidence sequence for $\theta^\star$:
	$$\mathcal\C_t = \left\{\theta\in\mathcal B(B)\,:\,\tfrac{1}{2}\|\theta-\wh\theta_t\|_{V_t}^2\leq \tfrac{dp}{2}\log\tfrac{(B+1)^2e\max(dp,t+d)}{dp} + 2\lambda B^2 + t\phi_d (B+1) + d\log\tfrac{d}{\delta}\right\}\,.$$
\end{corollary}
\begin{proof}
	Let $\beta_t^2 = dp\log\tfrac{(B+1)^2e\max(dp,t+d)}{dp} + 2t\phi_d (B+1) + 2d\log\tfrac{d}{\delta}$. From Theorem \ref{thm:confseq}, with probability at least $1-\delta$, uniformly for every $t$, $\|\theta^\star-\wh\theta_t\|_{\Lambda_t}^2\leq\beta_t^2$. Adding to both sides of this inequality $\frac{\lambda}{2}\|\theta^\star-\wh\theta_t\|_2^2$, and relaxing the RHS using that $\|\theta^\star-\wh\theta_t\|_2^2\leq 4 B^2$, we conclude.
\end{proof}




%% file: Regret_analysis.tex
\section{Regret bounds for Mixing-LinUCB}
\label{sec:bandit_regret}

In this section, we establish worst-case and gap-dependent cumulative regret bounds for mixing UCB algorithm (Mixing Lin-UCB). However, to account for the fact that Mixing-LinUCB selects actions with delays, the standard elliptical potential arguments must be modified. Throughout this section, we let $R_t = \siprod{\theta^{\star}}{X_t^{\star} - X_t}$ (where $X_t^\star = \argmax_{x\in\X_t}\langle \theta^\star, x\rangle$) denote the regret in round $t$, and $\beta_{t}^2 = dp\log\tfrac{(B+1)^2e\max(dp,t+d)}{dp} + 4\lambda B^2 + 2t\phi_d (B+1) + 2d\log\tfrac{d}{\delta}$ denote the squared radius of the ellipsoid $\mathcal{C}_t$ in \Cref{cor:conf}. 

\subsection{Worst-case regret bounds}

First, following the regret analysis in \citet{abbasi2011improved} (see also Section 19.3 in \citealp{lattimore2020bandit}), we upper bound the instantaneous regret. From our boundedness assumptions ($\theta^\star\in\mathcal B(B)$ and $\X_t\subseteq\mathcal B(1)$), we easily deduce that $R_t \leq 2B$. Under the event that our confidence sequence contains $\theta^{\star}$ at every step $t$, we have another bound on $R_t$. If we define $\wt\theta_{t-d} \in \mathcal{C}_{t-d}$ to be the point at which $\siprod{\wt\theta_{t-d}}{X_t} = \mathrm{UCB}_{\mathcal{C}_{t-d}}(X_t)$, then from the definition of $X_t$ we have
\begin{equation*}
\siprod{\theta^{\star}}{X_t^{\star}} \leq \max_{x \in \mathcal{X}_{t}}\max_{\theta \in \mathcal{C}_{t-d}}\siprod{\theta}{x} = \max_{x \in \mathcal{X}_{t}} \ucb_{\C_{t-d}}(x) = \ucb_{\C_{t-d}}(X_t)= \siprod{\wt\theta_{t-d}}{X_t}\,.
\end{equation*}

Recall that, for all $s$, $V_s = \Lambda_s + \lambda\mathrm{Id}$, which is invertible as $\lambda>0$. Thus, by Cauchy-Schwarz,
\begin{equation*}
R_t \leq \siprod{\wt\theta_{t-d} - \theta^{\star}}{X_t} \leq \|\wt\theta_{t-d} - \theta^{\star}\|_{V_{t-d}}\|X_t\|_{V_{t-d}^{-1}} \leq 2\beta_{t-d}\|X_t\|_{V_{t-d}^{-1}}\,.
\end{equation*}

This means that the instantaneous regret satisfies the bound
\begin{equation}
R_t \leq 2\max(B, \beta_{t-d})\min(1, \|X_t\|_{V_{t-d}^{-1}})\,.\label{eqn:inst_reg}
\end{equation}

Next, we separate the regret suffered in the first $d$ rounds and the remaining $T-d$ rounds. We then use Cauchy-Schwarz once more, and the fact that $\beta_{t}$ is increasing in $t$, to obtain
\begin{align*}
\Reg(T) &\leq 2dB + \sqrt{(T-d){\textstyle\sum_{t=d+1}^{T}}R_t^2}\\
&\leq 2dB + \sqrt{4(T-d)\max(B^2, \beta_{T-d}^2){\textstyle\sum_{t=d+1}^{T}}\min(1, \|X_t\|_{V_{t-d}^{-1}}^2)}\,.
\end{align*}

At this point, we must depart from the standard linear UCB analysis \citep{abbasi2011improved, lattimore2020bandit}. We bound the sum of the \emph{elliptical potentials} $\sum_{t=d+1}^{T}\min(1, \|X_t\|_{V_{t-d}^{-1}}^2)$ using the following variant of the well-known ``elliptical potential lemma'' (see Appendix), which accounts for the fact that the feature covariance matrix $V_{t-d}$ is updated with a delay of $d$ steps.
\begin{lemma}
For all $T \geq d+1$,
\begin{equation*}
\sum_{t=d+1}^{T}\min(1, \|X_t\|_{V_{t-d}^{-1}}^2) \leq 2dp\log(1 + \tfrac{T}{\lambda dp})\,.
\end{equation*}
\label{lem:delay_elliptical}
\end{lemma}

We can now state a worst-case regret upper bound for Mixing-LinUCB.

\begin{theorem}
\label{theorem:regret_bound}

Fix $\lambda = 1/B^2$, $d>0$ and $\delta\in(0,1)$. With probability at least $1 - \delta$, for all $T > d$, the regret of Mixing-LinUCB satisfies
\begin{align*}
\Reg(T) \leq 2dB + \sqrt{8dpT\max(B^2, \beta_{T}^2)\log(1 + \tfrac{B^2T}{dp})}\,.
\end{align*}
\label{thm:worst_case_reg}
\end{theorem}

From the definition of $\beta_T$, we see that this regret bound is of the order
\begin{equation*}
\Reg(T)  = \mathcal{O}\left(dB + dp\sqrt{T}\log\tfrac{TB}{dp} + T\sqrt{Bdp\phi_d\log\tfrac{TB}{dp}} + d\sqrt{pT\log\tfrac{TB}{p\delta}}\right)\,.
\end{equation*}

For any fixed (\emph{i.e.}, not depending on $T$) delay $d$, this regret bound is linear in $T$. To obtain meaningful regret bounds, it is therefore crucial to set $d$ as a function of $T$ and the rate at which the mixing coefficients decay to zero. We point out that if $T$ is unknown, one could probably use a more general framework where the delay is time dependent which might lead to non-trivial results, but we do not pursue this here. Under the assumption that the noise variables are either geometrically or algebraically mixing, we obtain the following worst-case regret bounds.

\begin{corollary}
\label{cor:geometric_mixing}
Suppose that the noise satisfies Assumption \ref{ass:mixing-subgaussianity} with $\phi_d=Ce^{-\frac{d}{\tau}}$ for some $C,\tau >0$ (\emph{geometric mixing}), and set $d = \lceil \tau\log\tfrac{BCT}{p}\rceil$. Then, the regret of Mixing-LinUCB satisfies
\begin{equation*}
\Reg(T) = \mathcal{O}\left(\tau p\sqrt{T}\left(\log\tfrac{TB\max(1,C)}{p}\right)^2 + p\sqrt{T\tau}\log\tfrac{TB\max(1,C)}{p} + \tau\log\tfrac{BCT}{p}\sqrt{pT\log\tfrac{TB}{p\delta}}\right)\,.
\end{equation*}
\end{corollary}

\begin{corollary}
Suppose that the noise satisfies Assumption \ref{ass:mixing-subgaussianity} with $\phi_d=Cd^{-r}$ for some $C>0$ and $r >0$ (\emph{algebraic mixing}), and set $d=\lceil CT^{1/(1+r)}\rceil$. Then, the regret of Mixing-LinUCB satisfies
\begin{equation*}
\Reg(T) =\mathcal{O}\left(CBT^{1/(1+r)} + CT^{\frac{3+r}{2(1+r)}}\left(p\log\tfrac{TB}{p} + \sqrt{Bp\log\tfrac{T^{r/(1+r)}B}{Cp}} + \sqrt{p\log\tfrac{TB}{p\delta}}\right)\right)\,.
\end{equation*}
\label{cor:algebraic_mixing}
\end{corollary}

Up to a factor of $\tau\log T$, the bound for geometrically mixing noise matches the regret bound for linear UCB with i.i.d.~noise. This bound is trivial for $r\leq 1$, however for $r>1$ we get sublinear regret, and in particular we recover standard rates up to logarithmic factors in the limit where $r \to \infty.$

\subsection{Gap-dependent regret bounds}

Under the assumption that, each round, the gap between the expected reward of the optimal arm and the expected reward of any other arm is at least $\Delta > 0$, we get regret bounds with better dependence on $T$. More precisely, define the \emph{minimum gap} $\Delta = \min_{t \in [T]}\min_{x \in \mathcal{X}_t: x \neq X_t^{\star}}\siprod{X_t^{\star} - x}{\theta^{\star}}$, and assume that $\Delta > 0$. Since we either have $R_t = 0$ or $R_t \geq \Delta > 0$, it follows that
\begin{equation*}
R_t \leq R_t^2/\Delta\,.
\end{equation*}

In our worst-case analysis, we showed that
\begin{equation*}
\sum_{t=d+1}^{T}R_t^2 \leq 8dp\max(B^2, \beta_T^2)\log(1 + \tfrac{T}{\lambda dp})\,.
\end{equation*}

Combined with the previous inequality, we obtain the following gap-dependent regret bound.

\begin{theorem}
Fix $\lambda = 1/B^2$, $d>0$, and $\delta\in(0,1)$. With probability at least $1 - \delta$, for all $T > d$, the regret of Mixing-LinUCB satisfies
\begin{align*}
\Reg(T) \leq 2dB + \frac{8dp}{\Delta}\max(B^2, \beta_{T}^2)\log\left(1 + \frac{B^2T}{dp}\right)\,.
\end{align*}
\label{thm:gap_reg}
\end{theorem}

Similarly to the worst-case bound in \Cref{thm:worst_case_reg}, for any fixed $d > 0$, this regret bound is linear in $T$. By setting $d$ as a suitable function of $T$, we obtain the following gap-dependent regret bounds under geometrically or algebraically mixing noise. 

\begin{corollary}
\label{cor2:geometric_mixing}
Suppose that the noise variables are geometrically mixing and set $d = \lceil \tau\log\tfrac{BCT}{p}\rceil$. Then the regret of Mixing-LinUCB satisfies
\begin{equation*}
\Reg(T) = \mathcal{O}\left(\frac{8\tau p}{\Delta}\left(\log\frac{BCT}{p}\right)^2\log\left(1 + \frac{B^2T}{p\tau \log\frac{BCT}{p}}\right)\left(\frac{p}{2}\log\frac{T}{p\tau}+\log \frac{\tau\log \frac{BCT}{p}}{\delta}\right)\right)\,.
\end{equation*}
\label{cor:gap_reg_geo}
\end{corollary}

\begin{corollary}
\label{cor2:algebraic_mixing}
Suppose that the noise variables are algebraically mixing and set $d=\lceil CT^{1/(1+r)}\rceil$. Then the regret of Mixing-LinUCB satisfies
\begin{equation*}
R(T) =\OO\left( \frac{8Cp}{\Delta}T^{\frac{2}{1+r}}\log\left(1 + \frac{B^2T}{pCT^{1/(1+r)}}\right)\left(\frac{p}{2}\log\frac{(B+1)^2eT}{p}+\log \frac{CT^{1/(1+r)}}{\delta}\right)\right)\,.
\end{equation*}
\label{cor:gap_reg_alg}
\end{corollary}


%% file: Conclusion.tex
\section{Conclusion}

We leave several interesting questions open for future research. Some of these are listed below.


An important limitation of our algorithm is that it requires the knowledge of the mixing coefficients (or at least an 
upper-bound on them). It would be interesting to explore the possibility of relaxing this assumption and to design an 
algorithm which infers the mixing coefficients while minimizing the regret. We note that the problem of estimating 
mixing coefficients is already a hard problem on its own right, with tight sample-complexity results only available in 
special cases such as Markov chains \citep{hsu2019mixing,wolfer2020mixing}. We also note that in order to recover the 
standard rate for the regret bound, the delay $d$ introduced in our algorithm need to be chosen as a function of the 
horizon $T$. We believe that this could be fixed at little conceptual expense by using time-varying delay in the 
analysis, but we did not attempt to work out the (potentially non-trivial) details here.

Another limitation is that our analysis assumed throughout that the adversary picking the decision sets $\X_t$ is 
oblivious, which is typically not required in linear bandit problems. For us, this was necessary to avoid potential 
statistical dependence between decision sets and the nonstationary observations. We believe that this issue can be 
handled at least for some classes of adversaries. For instance, it is easy to see that our analysis would carry through 
under the assumption that the decision sets be selected based on delayed information only. We leave the investigation 
of this question under more realistic assumptions open for future work.




%% file: Appendix.tex
\section{Technical Appendices and Supplementary Material}

\subsection[Proof of Proposition 2]{Proof of Proposition \ref{pro:uniform_ewa_regret}}\label{app:prop2}
For the EWA forecaster with prior $Q_1$, we can rewrite the regret via a standard telescoping argument (see Lemma B.1 in \citealp{clerico2025confidence}) as 
$$\Regret_t(\bar{\theta}) = -\log\int\exp\left(-\sum_{s=1}^{t}\ell_s(\theta) + \sum_{s=1}^{t}\ell_s(\bar{\theta})\right)\mathrm{d}Q_1(\theta)\,.$$
Using the variational representation of the KL divergence, this can be upper bounded as
\begin{align*}
\Regret_t(\bar{\theta}) &= \inf_{Q}\left\{\int\sum_{s=1}^{t}\ell_s(\theta)\mathrm{d}Q(\theta) - \sum_{s=1}^{t}\ell_s(\bar{\theta}) + D_{\mathrm{KL}}(Q||Q_1)\right\}\\
&\leq \inf_{c \in (0, 1]}\left\{\int\sum_{s=1}^{t}\ell_s(\theta)\mathrm{d}P_c(\theta) - \sum_{s=1}^{t}\ell_s(\bar{\theta}) + D_{\mathrm{KL}}(P_c||Q_1)\right\}\,,
\end{align*}

where $P_c$ is the uniform measure on the closed Euclidean ball of radius $c$ in $\real^p$, centred at $\bar\theta$. We remark that for all $c \in (0, 1]$, $P_c\ll Q_1$. Therefore, for all $c \in (0, 1]$,
\begin{equation*}
D_{\mathrm{KL}}(P_c||Q_1) = \int p\log\frac{B+1}{c}\mathrm{d}Q_1(\theta) = p\log\frac{B+1}{c}\,.
\end{equation*}

Taking a second-order Taylor expansion of the total squared loss around $\bar{\theta}$, and using the fact that the mean of $P_c$ is $\bar{\theta}$, we obtain
\begin{align*}
\sum_{s=1}^{t}\int_{\real^p}\big(\ell_s(\theta)-\ell_s(\bar{\theta})\big)\mathrm{d}P_c(\theta) &= \sum_{s=1}^t\int_{\real^p}\left(\iprod{\theta - \bar{\theta}}{\nabla\ell_s(\bar{\theta})} + \tfrac{1}{2}\iprod{\theta - \bar{\theta}}{X_s}^2\right)\mathrm{d}P_c(\theta) \leq \frac{tc^2}{2}\,,
\end{align*}
where we used that $\|X_s\|_2\leq 1$ for all $s$ in the last inequality. Combining everything so far, we obtain
\begin{equation*}
\Regret_t(\bar{\theta}) \leq \inf_{c \in (0, 1]}\left\{p\log\frac{B+1}{c} + \frac{tc^2}{2}\right\} \leq \frac{p}{2}\log\frac{(B+1)^2e\max(p, t)}{p}\,,
\end{equation*}
where the last term is obtained taking $c = \min(1, \sqrt{p/t})$.

\subsection[Proof of Lemma 1]{Proof of Lemma \ref{lemma:conc}}
Let $D_t = \ell_t(\theta^\star) - \mathcal L_t(Q_t)$ and $\lambda_t(\theta) = \langle \theta -\theta^\star, X_t\rangle$. It is easy to check that
$$D_t = \log\int e^{\lambda_t(\theta)\ee_t - \lambda_t(\theta)^2/2}\dd Q_t(\theta)\,.$$

Fix $i\in \{1,\dots, d\}$. We denote as $S^{(i)} = (S_k^{(i)})_{k\geq 1}$ the subsequence defined as $S_k^{(i)} = \sum_{j=1}^k D_{i+(j-1)d}$. We also define $\F^{(i)}_k = \F_{i+(k-1)d}$. It is easy to check that $(S_k^{(i)})_{k\geq 1}$ is adapted with respect to $(\F^{(i)}_k)_{k\geq 1}$. Now, let $M^{(i)}_k = \exp (S_k^{(i)} - (k-1)(2B+1)\phi_d)$. We will show that $(M^{(i)}_k)_{k\geq1}$ is a super-martingale with respect to $(\F^{(i)}_k)_{k\geq 1}$, with initial expectation bounded by $1$. For this it is enough to show that for any $k\geq 1$ we have $\E[e^{D_{i+(k-1)d}-(2B+1)\phi_d}|\F^{(i)}_{k-1}] \leq 1$. This is true for $k=1$ (where we let $\F^{(i)}_0$ be the trivial $\sigma$-field, or more generally a $\sigma$-field independent of the noise). Indeed, as $i\leq d$, $X_i$ is $\F_0$ measurable and hence independent of $\ee_i$. From Assumption \ref{ass:mixing-subgaussianity}, we know that $\ee_i$ is sub-Gaussian, and so $\E[e^{D_i}]\leq 1$. 

Let us now check the case $k\geq 2$. For convenience, we define $t_k^{(i)} = i + (k-1)d$. We note that $\F_{t_k^{(i)}} = \F_k^{(i)}$. We have
\begin{align*}\E[e&^{D_{i+(k-1)d}-(2B+1)\phi_d}|\F^{(i)}_{k-1}] \\&\qquad= \E\left[\int \exp\big(\lambda_{t_k^{(i)}}(\theta)\ee_{t_k^{(i)}} - \lambda_{t_k^{(i)}}(\theta)^2/2 - (2B+1)\phi_d\big)\mathrm dQ_{t_k^{(i)}}(\theta)\middle|\F_{k-1}^{(i)}\right]\,.\end{align*}
Now, $Q_{t_k^{(i)}}$ only depends on the noise up to $\ee_{t_k^{(i)}-d} = \ee_{t_{k-1}^{(i)}}$, thanks to the delayed bandit framework. Henceforth, we can swap the conditional expectation and the integral. In a similar way, we can bring $\exp\big( - \lambda_{t_k^{(i)}}(\theta)^2/2 - (2B+1)\phi_d\big)$ outside of the conditional expectation, as it is $\F_{k-1}^{(i)}$ measurable. We get 
\begin{align*}
	\E[e&^{D_{i+(k-1)d}-(2B+1)\phi_d}|\F^{(i)}_{k-1}] \\&= \int \E\left[\exp\big(\lambda_{t_k^{(i)}}(\theta)\ee_{t_k^{(i)}}\big) \middle|\F_{k-1}^{(i)}\right]\exp\big(- \lambda_{t_k^{(i)}}(\theta)^2/2 - (2B+1)\phi_d\big)\mathrm dQ_{t_k^{(i)}}(\theta)\\
	&\leq \int \exp\big(\lambda_{t_k^{(i)}}(\theta)^2/2 + \lambda_{t_k^{(i)}}\E[\ee_{t_k^{(i)}}|\F_{k-1}^{(i)}]\big)\exp\big(- \lambda_{t_k^{(i)}}(\theta)^2/2 - (2B+1)\phi_d\big)\mathrm dQ_{t_k^{(i)}}(\theta)\\
	&\leq \int \exp\big(\big|\lambda_{t_k^{(i)}}(\theta)\big|\phi_d - (2B+1)\phi_d\big)\mathrm dQ_{t_k^{(i)}}(\theta)\,,
\end{align*}
where the two inequalities use the sub-Gaussianity and mixing properties of Assumption \ref{ass:mixing-subgaussianity}. Now, by construction $Q_{t_k^{(i)}}$ has support on $\mathcal B(B+1)$, and for every $\theta\in\mathcal B(B+1)$ $$\big|\lambda_{t_k^{(i)}}(\theta)\big| \leq \|\theta-\theta^\star\|_2\|X_{t_k^{(i)}}\|_2\leq 2B+1\,,$$ where we also used that $\|X_{t_k^{(i)}}\|_2\leq 1$, as for all $t$ we are assuming that $\X_t\subseteq\mathcal B(1)$. We thus conclude that $(M_k^{(i)})_{k\geq 1}$ is indeed a super-martingale, non-negative and with initial value bounded by $1$. By Ville's inequality it follows that 
$$\mathbb P\big(S_k^{(i)}\leq k(2B+1)\phi_d +\log\tfrac{d}{\delta}\,,\;\forall k\geq 1\big)\geq 1-\tfrac{\delta}{d}\,.$$
Now that we have proven that we have a super-martingale for each block, the desired anytime valid concentration result follows directly from a simple union bound. 

\subsection[Proof of Lemma 2]{Proof of Lemma \ref{lemma:reg}}
Fix $t\geq 1$, and let $i\in\{1,\dots,d\}$ and $k\geq 1$ be such that $t = i + (k-1)d$. Let $I_j = \{j+d\mathbb N\}\cap \{1,\dots, t\}$, for $j\in\{1,\dots d\}$. We consider $d$ independent EWA forecaster (all initialised with the uniform prior on $\mathcal B(B+1)$). The $j$\textsuperscript{th} forecaster only acts and receive feedback from the rounds in $I_j$. We note that the $j$\textsuperscript{th} forecaster acts for $t_j$ rounds, where $t_j = k$ if $j\geq i$, and $t_j=k-1$ otherwise. We denote as $R^{(j)}$ the regret of the $j$\textsuperscript{th} forecaster (which only takes into account the losses at the rounds in $I_j$, with comparator $\bar\theta$. By Proposition \ref{pro:uniform_ewa_regret} we get 
$$\Regret_t(\bar\theta) = \sum_{j=1}^d R^{(j)}\leq \sum_{j=1}^d \frac{p}{2}\log\frac{(B+1)^2e\max(p,t_j)}{p}\,.$$
We conclude by noticing that, for all $j$, $t_j\leq (t+d)/d$.

\subsection{Proof of Lemma \ref{lem:delay_elliptical}}
\label{proof:delay_elliptical}
We recall the standard Elliptical Potential Lemma (see e.g.\ Lemma 11 in \citealp{abbasi2011improved}), which we will use in our proof of \Cref{lem:delay_elliptical}.

\begin{lemma}[Elliptical Potential Lemma]
\label{lemma:elliptical_potential}
Let $(X_t)_t$ be any sequence of vectors in $\mathbb{R}^p$ satisfying $\max_{t \in [T]}\|X_t\|_2 \leq L$ and let $V_T = \sum_{t=1}^{T}X_tX_t^{\top} + \lambda I$, for some $\lambda > 0$. Then
\begin{equation*}
\sum_{t=1}^{T}\min(1, \|X_t\|_{V_{t-1}^{-1}}^2) \leq 2p\log(1 + \tfrac{TL^2}{\lambda p})\,.
\end{equation*}
\end{lemma}

Next, we introduce some notation. For $t>d$, define $(i(t),k(t)) \in [d]\times[K]$ such that $t=i(t) +k(t)d$ and let $$V_{k(t)-1}^{i(t)}=\sum_{k=0}^{k(t)-1}X_{k}^{i(t)}(X_{k}^{i(t)})^{\top}+\lambda \mathrm{Id}\,, $$ 
where $X_{k}^{i(t)}=X_{i(t)+kd}$. With this notation, we can state the following lemma.

\begin{lemma}
    For any $t>d$, we have 
    \[V_{t-d} 
\succcurlyeq V_{k(t)-1}^{i(t)}\,, \]
which implies that $ \norm{X_t}_{V_{t-d}^{-1}}^2 \leq \norm{X_t}_{\pa{V_{k(t)-1}^{i(t)}}^{-1}}^2$ for any $t>d$.
\label{lem:loewner}
\end{lemma}
\begin{proof}
    Notice that we can write $V_{t-d} = \sum_{s=1}^{t-d}X_sX_s^{\top} + \lambda\mathrm{Id}=V_{k(t)}^{i(t)}+\sum_{s=1,s\notin S_t}^{t-d}X_sX_s^{\top}$ 
where $S_t := \{ s = i(t) + (k-1)d, k \in [k(t)] \}$ is the set of indices $(i(t), i(t)+d,\dots, i(t)+(k(t)-1)d)$. The statement now follows from the fact that $\sum_{s=1,s\notin S_t}^{t-d}X_sX_s^{\top} \succcurlyeq 0$. 
\end{proof}

We are now ready to prove Lemma \ref{lem:delay_elliptical}. For now, let us assume that $T = Kd$, for some $K > 1$. Using \Cref{lem:loewner} and then \Cref{lemma:elliptical_potential}, we have
\begin{align*}
\sum_{t=d+1}^{T}\min(1, \|X_t\|_{V_{t-d}^{-1}}^2) &\leq \sum_{t=d+1}^{T}\min(1, \|X_{k(t)}^{i(t)}\|_{(V_{k(t)-1}^{i(t)})^{-1}}^2)\\
&= \sum_{i=1}^{d}\sum_{k=1}^{K-1}\min(1, \|X_{k}^{i}\|_{(V_{k-1}^{i})^{-1}}^2)\\
&\leq 2dp\log(1 + \tfrac{(K-1)L^2}{\lambda p})\,.
\end{align*}

One can verify that if $T$ is not divisible by $d$, the above inequality still holds if we replace $K$ by $\lceil\frac{T}{d}\rceil$. Therefore, regardless of whether $T$ is divisible by $d$, we have
\begin{equation*}
\sum_{t=d+1}^{T}\min(1, \|X_t\|_{V_{t-d}^{-1}}^2) \leq 2dp\log(1 + \tfrac{TL^2}{\lambda dp})\,.
\end{equation*}

This concludes the proof of \Cref{lem:delay_elliptical}.

%
\subsection{Proof of Corollary \ref{cor:geometric_mixing} and Corollary \ref{cor:algebraic_mixing}}
We start by recalling the general result
\begin{equation}
\label{eq:general_regret}
\Reg(T)  = \mathcal{O}\left(\underbrace{dB}_{(1)} + \underbrace{dp\sqrt{T}\log\tfrac{TB}{dp}}_{(2)} + \underbrace{T\sqrt{Bdp\phi_d\log\tfrac{TB}{dp}}}_{(3)} + \underbrace{d\sqrt{pT\log\tfrac{TB}{p\delta}}}_{(4)}\right)\,.
\end{equation}

To simplify the following calculations, we do not force $d$ to be a positive integer. One can always round $d$ without changing the rates of the regret bounds.

\textbf{Geometric Mixing:}

Assume $d= \tau \log \frac{BCT}{p}$. We notice that the term $(1)$ is logarithmic in $T$ and thus negligible. From the definition of geometric mixing, it holds that $\phi_d = Ce^{-\frac{d}{\tau}}=\frac{p}{BT}.$ Therefore, \[(3) \leq p\sqrt{\tau  T}\log\frac{TB}{p}\,.\] Substituting the value of $d$ yields the desired bounds for terms $(2)$ and $(4)$ in Equation \ref{eq:general_regret}, and hence the desired statement.

\textbf{Algebraic mixing:}

Assume $d =  CT^{\frac{1}{1+r}}$, we notice that in this case since $\phi_d = Cd^{-r},$ we have $d\phi_d = Cd^{1-r}$. In particular this implies that $T\sqrt{d\phi_d} = T^{\frac{3+r}{2(1+r)}}$ and thus
\[(3)\leq C\sqrt{Bp\log\frac{TB}{p}}T^{\frac{3+r}{2(1+r)}}\]
The same way $(2)$ and $(4)$ are of order $d\sqrt{T}= T^{\frac{3+r}{2(1+r)}}$ and replacing in Equation \ref{eq:general_regret} yields the desired statement.

\subsection{Proof of Corollary \ref{cor:gap_reg_geo} and Corollary \ref{cor:gap_reg_alg}}
\label{proof:cor2}
We start by recalling the general result
\begin{align*}
\Reg(T) \leq 2dB + \frac{8dp}{\Delta}\max(B^2, \beta_{T}^2)\log\left(1 + \frac{B^2T}{dp}\right)\,,
\end{align*}
where $\beta_T^2 = \underbrace{dp\log\tfrac{(B+1)^2e\max(dp,T+d)}{dp}}_{(1)} + \underbrace{2T\phi_d (B+1)}_{(2)} + \underbrace{2d\log\tfrac{d}{\delta}}_{(3)}$.

\textbf{Geometric Mixing:}

Assume $d= \tau \log \frac{BCT}{p}$, then $(2)=\frac{2p(B+1)}{BC}$ is a constant.  Hence we have \[ \Reg(T)\leq 2dB + \frac{8d^2p}{\Delta}\left(p\log\tfrac{(B+1)^2e\max(dp,T+d)}{dp}+2\log \frac{d}{\delta}+ \frac{2p(B+1)}{dBC} \right)\log\left(1 + \frac{B^2T}{dp}\right), \]
which under the assumption that $\beta_T \geq B$ and replacing $d$ by its definition yields
\begin{align*}
\Reg(T) &\leq 2B\tau \log \frac{BCT}{p}\\
&+ \frac{8\tau^2p}{\Delta}\log\left(1 + \frac{B^2T}{p\tau \log \frac{BCT}{p}}\right)\left(\left(\log\frac{BCT}{p}\right)^2 \left(p\log\tfrac{(B+1)^2eT}{p}+2\tau\frac{\log \frac{BCT}{p}}{\delta}\right)+\frac{2p(B+1)}{BC}\right)\,.
\end{align*}

If $\Delta$ is constant, then for large $T$, the first term and the constant part coming from $(2)$ become negligible. Therefore,
\[\Reg(T) = \OO\left(\frac{8\tau^2p}{\Delta}\log\left(1 + \frac{B^2T}{p\tau \log \frac{BCT}{p}}\right)\left(\log\frac{BCT}{p}\right)^2 \left(\frac{p}{2}\log\tfrac{(B+1)^2eT}{p}+\tau\frac{\log \frac{BCT}{p}}{\delta}\right)\right)\]

\textbf{Algebraic mixing:}

Assume $d =  CT^{\frac{1}{1+r}}$, then we have
\[\beta_T^2 \leq CT^{\frac{1}{1+r}}\log \frac{(B+1)^2eT}{p} + 2C(B+1)T^{\frac{2}{1+r}} + 2C T^{\frac{1}{1+r}}\log\frac{C T^{\frac{1}{1+r}}}{\delta}.\]

Under the regime where $2dB \le \frac{8dp}{\Delta}\max(B^2, \beta_{T}^2)\log\left(1 + \frac{B^2T}{dp}\right)$ and $B\le \beta_T$ this leads to 
\[\Reg(T) =\OO\left( \frac{8Cp}{\Delta}T^{\frac{2}{1+r}}\log\left(1 + \frac{B^2T}{pCT^{1/(1+r)}}\right)\left(\frac{p}{2}\log\frac{(B+1)^2eT}{p}+\log \frac{CT^{1/(1+r)}}{\delta}\right)\right)\,.\]

%% file: neurips_2025.bbl
\begin{thebibliography}{27}
\providecommand{\natexlab}[1]{#1}
\providecommand{\url}[1]{\texttt{#1}}
\expandafter\ifx\csname urlstyle\endcsname\relax
  \providecommand{\doi}[1]{doi: #1}\else
  \providecommand{\doi}{doi: \begingroup \urlstyle{rm}\Url}\fi

\bibitem[Abe and Long(1999)]{abe1999associative}
Naoki Abe and Philip~M. Long.
\newblock Associative reinforcement learning using linear probabilistic concepts.
\newblock In \emph{Proceedings of the Sixteenth International Conference on Machine Learning}, 1999.

\bibitem[Auer(2003)]{auer2003using}
Peter Auer.
\newblock Using confidence bounds for exploitation-exploration trade-offs.
\newblock \emph{J. Mach. Learn. Res.}, 3:\penalty0 397–422, 2003.

\bibitem[Abe et~al.(2003)Abe, Biermann, and Long]{abe2003reinforcement}
Naoki Abe, Alan~W. Biermann, and Philip~M. Long.
\newblock Reinforcement learning with immediate rewards and linear hypotheses.
\newblock \emph{Algorithmica}, 37\penalty0 (4):\penalty0 263–293, 2003.

\bibitem[Li et~al.(2010)Li, Chu, Langford, and Schapire]{li2010contextual}
Lihong Li, Wei Chu, John Langford, and Robert~E Schapire.
\newblock A contextual-bandit approach to personalized news article recommendation.
\newblock In \emph{Proceedings of the 19th international conference on World wide web}, pages 661--670, 2010.

\bibitem[Korkut and Li(2021)]{korkut2021disposable}
Melda Korkut and Andrew Li.
\newblock Disposable linear bandits for online recommendations.
\newblock \emph{Proceedings of the AAAI Conference on Artificial Intelligence}, 35\penalty0 (5), 2021.

\bibitem[Cohen et~al.(2020)Cohen, Lobel, and Paes~Leme]{cohen2020feature}
Maxime~C Cohen, Ilan Lobel, and Renato Paes~Leme.
\newblock Feature-based dynamic pricing.
\newblock \emph{Management Science}, 66\penalty0 (11):\penalty0 4921--4943, 2020.

\bibitem[Lai and Robbins(1985)]{lai1985asymptotically}
T.L. Lai and Herbert Robbins.
\newblock Asymptotically efficient adaptive allocation rules.
\newblock \emph{Advances in Applied Mathematics}, 6\penalty0 (1):\penalty0 4--22, 1985.

\bibitem[Abbasi-Yadkori et~al.(2011)Abbasi-Yadkori, P{\'a}l, and Szepesv{\'a}ri]{abbasi2011improved}
Yasin Abbasi-Yadkori, D{\'a}vid P{\'a}l, and Csaba Szepesv{\'a}ri.
\newblock Improved algorithms for linear stochastic bandits.
\newblock \emph{Advances in neural information processing systems}, 24, 2011.

\bibitem[Lattimore and Szepesv{\'a}ri(2020)]{lattimore2020bandit}
Tor Lattimore and Csaba Szepesv{\'a}ri.
\newblock \emph{Bandit algorithms}.
\newblock Cambridge University Press, 2020.

\bibitem[Flynn et~al.(2023)Flynn, Reeb, Kandemir, and Peters]{flynn2023improved}
Hamish Flynn, David Reeb, Melih Kandemir, and Jan~R Peters.
\newblock Improved algorithms for stochastic linear bandits using tail bounds for martingale mixtures.
\newblock \emph{Advances in Neural Information Processing Systems}, 36:\penalty0 45102--45136, 2023.

\bibitem[Bradley(2005)]{bradley2005basic}
Richard~C. Bradley.
\newblock Basic properties of strong mixing conditions: A survey and some open questions.
\newblock \emph{Probability Surveys}, 2:\penalty0 107--144, 2005.

\bibitem[Mohri and Rostamizadeh(2008)]{mohri2008rademacher}
M.~Mohri and A.~Rostamizadeh.
\newblock Rademacher complexity bounds for non-i.i.d. processes.
\newblock \emph{NeurIPS}, 2008.

\bibitem[Ab\'el\`es et~al.(2025)Ab\'el\`es, Clerico, and Neu]{abeles2024generalization}
Baptiste Ab\'el\`es, Eugenio Clerico, and Gergely Neu.
\newblock Generalization bounds for mixing processes via delayed {online-to-PAC} conversions.
\newblock In \emph{Proceedings of The 36th International Conference on Algorithmic Learning Theory}, 2025.

\bibitem[Jun et~al.(2017)Jun, Bhargava, Nowak, and Willett]{jun2017scalable}
Kwang-Sung Jun, Aniruddha Bhargava, Robert Nowak, and Rebecca Willett.
\newblock Scalable generalized linear bandits: Online computation and hashing.
\newblock In \emph{Advances in Neural Information Processing Systems}, volume~30, 2017.

\bibitem[Lee et~al.(2024)Lee, Yun, and Jun]{lee2024improved}
Junghyun Lee, Se-Young Yun, and Kwang-Sung Jun.
\newblock Improved regret bounds of (multinomial) logistic bandits via regret-to-confidence-set conversion.
\newblock In \emph{Proceedings of the 27th International Conference on Artificial Intelligence and Statistics}, pages 4474--4482, 2024.

\bibitem[Clerico et~al.(2025)Clerico, Flynn, Kotłowski, and Neu]{clerico2025confidence}
Eugenio Clerico, Hamish Flynn, Wojciech Kotłowski, and Gergely Neu.
\newblock Confidence sequences for generalized linear models via regret analysis, 2025.
\newblock URL \url{https://arxiv.org/abs/2504.16555}.

\bibitem[Vernade et~al.(2020{\natexlab{a}})Vernade, Carpentier, Lattimore, Zappella, Ermis, and Brueckner]{vernade2020linear}
Claire Vernade, Alexandra Carpentier, Tor Lattimore, Giovanni Zappella, Beyza Ermis, and Michael Brueckner.
\newblock Linear bandits with stochastic delayed feedback.
\newblock In \emph{International Conference on Machine Learning}, pages 9712--9721. PMLR, 2020{\natexlab{a}}.

\bibitem[Howson et~al.(2023)Howson, Pike-Burke, and Filippi]{howson2023delayed}
Benjamin Howson, Ciara Pike-Burke, and Sarah Filippi.
\newblock Delayed feedback in generalised linear bandits revisited.
\newblock In \emph{International Conference on Artificial Intelligence and Statistics}, pages 6095--6119. PMLR, 2023.

\bibitem[Garivier and Moulines(2008)]{garivier2008upper}
Aur{\'e}lien Garivier and Eric Moulines.
\newblock On upper-confidence bound policies for non-stationary bandit problems.
\newblock \emph{arXiv preprint arXiv:0805.3415}, 2008.

\bibitem[Russac et~al.(2019)Russac, Vernade, and Capp{\'e}]{russac2019weighted}
Yoan Russac, Claire Vernade, and Olivier Capp{\'e}.
\newblock Weighted linear bandits for non-stationary environments.
\newblock \emph{Advances in Neural Information Processing Systems}, 32, 2019.

\bibitem[Vernade et~al.(2020{\natexlab{b}})Vernade, Gyorgy, and Mann]{vernade2020non}
Claire Vernade, Andras Gyorgy, and Timothy Mann.
\newblock Non-stationary delayed bandits with intermediate observations.
\newblock In \emph{International Conference on Machine Learning}, pages 9722--9732. PMLR, 2020{\natexlab{b}}.

\bibitem[Cesa-Bianchi and Lugosi(2006)]{cesabianchi2006prediction}
Nicolò Cesa-Bianchi and Gabor Lugosi.
\newblock \emph{Prediction, Learning, and Games}.
\newblock Cambridge University Press, USA, 2006.

\bibitem[Gr\"{u}nwald(2007)]{grunwald2007minimum}
Peter~D. Gr\"{u}nwald.
\newblock \emph{The Minimum Description Length Principle (Adaptive Computation and Machine Learning)}.
\newblock The MIT Press, 2007.

\bibitem[Yu(1994)]{yu1994rates}
Bin Yu.
\newblock Rates of convergence for empirical processes of stationary mixing sequences.
\newblock \emph{The Annals of Probability}, 22\penalty0 (1):\penalty0 94--116, 1994.

\bibitem[Weinberger and Ordentlich(2002)]{weiberger2002delay}
M.J. Weinberger and E.~Ordentlich.
\newblock On delayed prediction of individual sequences.
\newblock \emph{IEEE Transactions on Information Theory}, 48\penalty0 (7), 2002.

\bibitem[Hsu et~al.(2019)Hsu, Kontorovich, Levin, Peres, Szepesv{\'a}ri, and Wolfer]{hsu2019mixing}
Daniel Hsu, Aryeh Kontorovich, David~A Levin, Yuval Peres, Csaba Szepesv{\'a}ri, and Geoffrey Wolfer.
\newblock Mixing time estimation in reversible markov chains from a single sample path.
\newblock \emph{The Annals of Applied Probability}, 29\penalty0 (4):\penalty0 2439--2480, 2019.

\bibitem[Wolfer(2020)]{wolfer2020mixing}
Geoffrey Wolfer.
\newblock Mixing time estimation in ergodic markov chains from a single trajectory with contraction methods.
\newblock In \emph{Algorithmic Learning Theory}, pages 890--905, 2020.

\end{thebibliography}
